\pdfoutput=1
\documentclass{article}
\usepackage{times}
\usepackage{graphicx}
\usepackage{algorithm2e}
\usepackage{amsthm}
\usepackage{amssymb, amsmath}

\newtheorem{thm}{\textsc{Theorem}}
\newtheorem{prop}{\textsc{Proposition}}
\newtheorem{lem}{\textsc{Lemma}}

\theoremstyle{definition}
\newtheorem{rem}{\textsc{Remark}}

\DeclareMathOperator{\argmax}{argmax}

\newcommand{\RR}{\mathbb{R}}

\newcommand{\ZZ}{\mathbb{Z}}

\newcommand{\trn}{\mathrm{T}}

\newcommand{\ex}[2][{}]{\mathbf{E}_{#1}\left[#2\right]}

\newcommand{\brl}{\left\{}
\newcommand{\brr}{\right\}}
\newcommand{\bl}{\left(}
\newcommand{\setslc}[2]{#1^{(1:#2)}}
\newcommand{\br}{\right)}

\newcommand{\ground}{\mathcal{N}}
\newcommand{\ucb}{\mathrm{ucb}}

\newcommand{\regret}[2]{\mathrm{Reg}_{#1}\bl #2 \br}
\newcommand{\assumedbound}{\nu}
\newcommand{\lemsupnum}{$1$}

\newcommand{\algname}{AFSM-UCB}
\newcommand{\algnamei}{GM-UCB}
\newcommand{\regretmcs}{\mathrm{Reg}_{\alpha, \mathrm{MCS}}}
\newcommand{\regretlsb}{\mathrm{Reg}_{\alpha, \mathrm{LSB}}}
\newcommand{\propevent}{\mathcal{F}}
\usepackage{natbib}

\newcommand{\myciteauthoryear}{\citet}
\newcommand{\parencite}{\citep}

\newcommand{\fx}{Fuji Xerox Co., Ltd.}
\title{Submodular Bandit Problem Under Multiple Constraints}
\author{Sho Takemori$^{*}$ \and Masahiro Sato$^{*}$ \and Takashi Sonoda$^{*}$
 \and Janmajay Singh$^{*}$ \and Tomoko Ohkuma$^{*}$}
\date{\normalsize{$^*$\fx}}
\begin{document}
\bibliographystyle{plainnat}
\maketitle

\begin{abstract}
   The \textit{linear submodular bandit problem}
   was proposed
   to simultaneously address diversified retrieval and online learning in a recommender system.
   If there is no uncertainty, this problem is equivalent to
   a submodular maximization problem under a cardinality constraint.
   However, in some situations, recommendation lists should satisfy
   additional constraints such as \textit{budget constraints},
   other than a cardinality constraint.
   Thus, motivated by diversified retrieval considering budget constraints,
   we introduce a submodular bandit problem under the intersection of
   $l$ \textit{knapsacks} and a \textit{$k$-system constraint}.
   Here $k$-system constraints form
   a very general class of constraints including
   cardinality constraints and the intersection of $k$ \textit{matroid} constraints.
   To solve this problem, we propose a non-greedy algorithm that adaptively focuses on
   a standard or modified upper-confidence bound.
   We provide a high-probability upper bound of an \textit{approximation regret},
   where the approximation ratio matches that of a fast offline algorithm.
   Moreover, we perform experiments
   under various combinations of constraints using a
   synthetic and two real-world datasets
   and demonstrate that our proposed methods outperform the existing baselines.
\end{abstract}

\section{INTRODUCTION}
The \textit{multi-armed bandit (MAB) problem}
has been widely used for practical applications.
Examples include
interactive recommender systems, Internet advertising,
portfolio selection, and clinical trials.
In a typical MAB problem, the agent selects one arm in each round. However,
in practice, it is more convenient to select more than one arm in each round.
Such a problem is called a
\textit{combinatorial bandit problem}
\parencite{chen2013combinatorial}.
For example, in \parencite{yue2011linear,radlinski2008learning},
they considered the problem where
in each round, the agent proposes multiple news articles or web documents to a user.

When recommending multiple items to a user,
agents should select \textit{well-diversified} items to maximize coverage of
the information the user finds interesting \parencite{yue2011linear}
or to reduce item similarity in the list \parencite{ziegler2005improving}.
Recommending redundant items leads to \textit{diminishing returns}
in terms of utility \parencite{yu2016linear}.
It is well-known that
properties such as diversity or diminishing returns
are well captured by \textit{submodular set functions} \parencite{krause2014submodular}.
To simultaneously address diversified retrieval and online learning in a recommender system,
\myciteauthoryear{yue2011linear} proposed a combinatorial bandit problem
(or more specifically a semi-bandit problem),
called the \textit{linear submodular bandit problem},
where in each round a sequence rewards are generated
by an unknown submodular function.

For a real-world application,
recommendation lists should satisfy several constraints.
We explain this by using a news article recommendation example.
For a comfortable user experience while selecting news articles from a recommendation list,
the length of the list should not be excessively long,
which implies that
the list should satisfy a cardinality constraint.
Furthermore, a user may not wish to spend more than a certain amount of time by reading news articles.
This can be modeled as a \textit{knapsack constraint}.
With only a knapsack constraint,
a system can recommend a long list of short (or low cost) news articles.
However, due to the space constraint of the web site, such a list cannot be displayed.
Therefore, it is necessary to consider a submodular bandit problem
under the intersection of the knapsack and cardinality constraints.

\myciteauthoryear{yue2011linear}
introduced a submodular bandit problem under a cardinality constraint
and proposed an algorithm called LSBGreedy.
Later, \myciteauthoryear{yu2016linear}
considered a submodular bandit problem under a knapsack constraint
and proposed two greedy algorithms called MCSGreedy and CGreedy.
However, such existing algorithms fail to properly optimize the objective function
under complex constraints.
In fact, we theoretically and empirically show that such simple greedy algorithms can perform poorly.

Under a simple constraint such as a cardinality or a knapsack constraint,
there is a simple rule to select elements.
This rule is called
the \textit{upper confidence bound (UCB) rule}
or the \textit{modified UCB rule}
if
the constraint is a cardinality or a knapsack constraint, respectively \parencite{yu2016linear}.
For example, with the UCB rule, the algorithm selects the element with the largest UCB
sequentially in each round.
Considering that our problem is a generalization of both the problems,
we should generalize both the rules.

In this study,
we solve the problem
under a more generalized constraint, i.e.,
the intersection of $l$ \textit{knapsacks} and \textit{$k$-system constraints}.
Here, the $k$-system constraints form a very general class of constraints, including
cardinality constraints and the intersection of $k$ \textit{matroid} constraints.
For example, when recommending news articles,
we can restrict the number of news articles from each topic with a $k$-system constraint.
To solve the problem, we propose a non-greedy algorithm that adaptively focuses on the UCB
and modified UCB rules.
Since the submodular maximization problem is NP-hard,
we theoretically evaluate our method by
an \textit{$\alpha$-approximation regret}, where $\alpha \in (0, 1)$ is an
approximation ratio.
In this study, we provide an upper bound of the $\alpha$-approximation regret in the case when
$\alpha = \frac{1}{(1 + \varepsilon)(k + 2l + 1)}$,
where $\varepsilon$ is a parameter of the algorithm.
We note that the approximation ratio matches that of an \textit{offline} algorithm
\parencite{badanidiyuru2014fast}.
To the best of our knowledge,
no known offline algorithm achieves a better approximation ratio than $\alpha$ above
and better computational complexity than the offline algorithm, simultaneously
\footnote{After we submitted this paper to the conference,
\myciteauthoryear{li2020efficient} have updated their preprint.
They proposed an offline submodular maximization algorithm and
improved the approximation ratio of \parencite{badanidiyuru2014fast} to
$1/(k + \frac{7}{4}l + 1) - \varepsilon$.
}.
\label{arxivfootnote}
More precisely, our contributions are stated as follows:

\subsection*{OUR CONTRIBUTIONS}
\begin{enumerate}
   \item We propose a submodular bandit problem with semi-bandit feedback under the intersection of
   $l$ knapsacks and $k$-system constraints (Section \ref{sec:prob-form}).
   This is the first attempt to solve the submodular bandit problem under such complex constraints.
   The problem is new even when the $k$-system constraint is a cardinality constraint.

   \item We propose a novel algorithm called
   \textit{\algname{}}
   that Adaptively Focuses on a Standard or Modified Upper Confidence Bound (Section \ref{sec:algo}).

   \item We provide a high-probability upper bound
   of an approximation regret for \algname{} (Section \ref{sec:main}).
   We prove that the $\alpha$-approximation regret
   $\regret{\alpha}{T}$ is given by $O(\sqrt{mT}\ln( mT/\delta ))$
   with probability in least $1 - \delta$
   and the computational complexity in each round is given as
   $O(m|\ground| \ln |\ground|/ \ln (1 + \varepsilon))$,
   where $\alpha = \frac{1}{(1 + \varepsilon)(k + 2l + 1)}$,
   $\varepsilon$ is a parameter of the algorithm,
   $T$ is the time horizon,
   $m$ is the cardinality of a maximal feasible solution, and
   $\ground$ is the ground set (e.g., the set of all news articles in the news recommendation example).
   We note that no known offline fast\footnote{We refer to Section \ref{sec:related-submod} for the meaning of ``fast''.} algorithm
   achieves a better approximation ratio than above\footnote{See footnote in page \pageref{arxivfootnote}.}.

   \item
   We empirically prove the effectiveness of our proposed method
   by comprehensively evaluating it on a synthetic and two real-world datasets.
   We show that our proposed method outperforms the existing greedy baselines such as LSBGreedy and CGreedy.
\end{enumerate}

\section{RELATED WORK}
\label{sec:related}
\subsection{SUBMODULAR MAXIMIZATION}
\label{sec:related-submod}
Although submodular maximization has been studied over four decades,
we introduce only recent results relevant to our work.
\myciteauthoryear{badanidiyuru2014fast} provided a maximization algorithm
for a non-negative, monotone submodular function with $l$ knapsack constraints
and a $k$-system constraint that achieves $\frac{1}{(1 + \varepsilon)(k + 2l + 1)}$-approximation solution.
Based on this work and \myciteauthoryear{gupta2010constrained},
\myciteauthoryear{mirzasoleiman2016fast} proposed a maximization algorithm called FANTOM
under the same constraint
in the case when the objective function is not necessarily monotone.
Our proposed method is inspired by these two offline algorithms.
However, because of uncertainty due to semi-bandit feedback,
we need a nontrivial modification.
A key feature of our method and aforementioned two offline algorithms
is that they filter out ``bad'' elements via a threshold.
Such a threshold method is also used for other problem settings
such as streaming submodular maximization under a cardinality constraint
\parencite{badanidiyuru2014streaming}.
Some algorithms \parencite{sarpatwar2019constrained,chekuri2010dependent,chekuri2014submodular}
achieves better approximation ratios than that of \parencite{badanidiyuru2014fast}
under narrower classes of constraints (e.g., a matroid + $l$ knapsacks).
However, these algorithms are not ``fast'' because
their computational complexity is $O(\mathrm{poly}(|\ground|))$ with a polynomial of high degree,
while that of \parencite{badanidiyuru2014fast} is $O(\frac{|\ground|}{\varepsilon^2}
\ln^2 \frac{|\ground|}{\varepsilon})$.
For example, the computational complexity of the algorithm provided in \parencite{sarpatwar2019constrained}
is $\widetilde{O}(|\ground|^6)$ when $k = 1$.
We refer to \parencite{sarpatwar2019constrained,mirzasoleiman2016fast} for further comparison with respect
to an approximation ratio and computational complexity.

\subsection{SUBMODULAR BANDIT PROBLEMS}
\myciteauthoryear{yue2011linear} introduced the linear submodular bandit problem
to solve a diversification problem in a retrieval system
and proposed a greedy algorithm called LSBGreedy.
Later, \myciteauthoryear{yu2016linear} considered a variant of the problem, that is,
the linear submodular bandit problem with a knapsack constraint
and proposed two greedy algorithms called MCSGreedy and CGreedy.
\myciteauthoryear{chen2017interactive} generalized the linear submodular bandit problem
to an infinite dimensional case, i.e.,
in the case where the marginal gain of the score function belongs to a \textit{reproducing kernel Hilbert space
(RKHS)}
and has a bounded norm in the space.
Then, they proposed a greedy algorithm called SM-UCB.
Recently, \myciteauthoryear{hiranandanicascading} studied a model combining
linear submodular bandits with a \textit{cascading model}
\parencite{craswell2008experimental}.
Strictly speaking, their objective function is not a submodular function.
Table \ref{tab:comparison} shows a comparison with other submodular bandit problems
with respect to constraints.

\begin{table}[ht]
   \label{tab:comparison}
   \centering
   \caption{Comparison of other submodular bandit algorithms with respect to constraints.}
   \begin{tabular}{cccc}
      Methods & Cardinality & Knapsack & $k$-system\\
      \hline \hline
      LSBGreedy & \checkmark & & \\
      \hline
      CGreedy & & \checkmark & \\
      \hline
      SM-UCB & \checkmark & & \\
      \hline
      Our method & \checkmark & \checkmark & \checkmark
   \end{tabular}
\end{table}

\section{DEFINITION}
\label{sec:def}
In this section, we provide definitions of terminology used in this paper.
Throughout this paper,
we fix a finite set $\ground$ called a ground set
that represents the set of the entire news articles in the news article recommendation example.
\subsection{SUBMODULAR FUNCTION}
In this subsection, we define submodular functions.
We refer to \parencite{krause2014submodular} for an introduction to this subject.

We denote by $2^{\ground}$ the set of subsets of $\ground$.
For $e \in \ground$ and $S \subseteq \ground$, we write $S + e = S \cup \brl e \brr$.
Let $f: 2^{\ground} \rightarrow \RR$ be a set function.
 We call $f$ a \textit{submodular function}
 if $f$ satisfies
 \begin{math}
     \Delta f(e|A) \ge \Delta f(e|B)
 \end{math}
 for any $A, B \in 2^{\ground}$ with $A \subseteq B$ and
 for any $e \in \ground \setminus B$.
 Here, $\Delta f(e | A)$ is the marginal gain when $e$ is added to $A$ and defined as $f(A + e) - f(A)$.
 We note that a linear combination of submodular functions with non-negative coefficients is also submodular.
A submodular function $f$ on $2^{\ground}$ is called monotone if
 $f(B) \ge f(A)$ for any $A, B \in 2^{\ground}$ with $A \subseteq B$.
A set function $f$ on $2^{\ground}$ is called non-negative if $f(S) \ge 0$ for any
$S \subseteq \ground$.
Although non-monotone submodular functions have important applications \parencite{mirzasoleiman2016fast},
we consider only non-negative, monotone submodular functions in this study
as in the preceding work \parencite{yue2011linear,yu2016linear,chen2017interactive}.

Many useful and interesting functions satisfy submodularity.
Examples include coverage functions,
probabilistic coverage functions \parencite{el2009turning},
entropy and mutual information \parencite{krause2012near} under an assumption
and ROUGE \parencite{lin2011class}.

\subsection{MATROID, $k$-SYSTEM, AND KNAPSACK CONSTRAINTS}
For succinctness, we omit formal definitions of
the matroid and $k$-system.
Instead, we introduce examples of matroids and remark that
the intersection of $k$ matroids is a $k$-system.
For definitions of these notions, we refer to \parencite{calinescu2011maximizing}.

First, we provide an important example of a matroid.
Let $\ground_{i} \subseteq \ground$ ($i=1,\dots, n$) be a partition of $\ground$, that is $\ground$
is the disjoint union of these subsets.
For $1 \le i \le n$, we fix a non-negative integer $d_i$ and let
$\mathcal{P} = \brl S \in 2^{\ground} \mid |S \cap \ground_i| \le d_i,\ \forall i\brr$.
Then, the pair $(\ground, \mathcal{P})$ is an example of
a matroid and called a \textit{partition matroid}.
Let $d$ be a non-negative integer and put
$\mathcal{U} = \brl S \in 2^{\ground} \mid |S| \le d \brr$.
Then $(\ground, \mathcal{U})$ is a special case of partition matroids
and called a \textit{uniform matroid}.
Let $(\ground, \mathcal{M}_i)$ for $1 \le i \le k$ be $k$ matroids, where $\mathcal{M}_i \subseteq 2^\ground$.
The intersection of matroids $(\ground, \cap_{i=1}^{k} \mathcal{M}_i)$
is not necessarily a matroid but a \textit{$k$-system}
(or more specifically it is a $k$-extendible system)
\parencite{calinescu2011maximizing,mestre2006greedy,mestre2015intersection}.
In particular, any matroid is a $1$-system.
For a $k$-system $(\ground, \mathcal{I})$ with $\mathcal{I} \subseteq 2^{\ground}$
and a subset $S \subseteq \ground$,
we say that $S$ satisfies the $k$-system constraint if and only
if $S \in \mathcal{I}$.
Trivially, a uniform matroid constraint is equivalent to
a cardinality constraint.

Next, we provide a definition of knapsack constraint.
Let $c: \ground \rightarrow \RR_{> 0}$ be a function.
For $e \in \ground$, we suppose $c(e)$ represents the cost of $e$.
Let $b \in \RR_{> 0}$ be a budget and $S \subseteq \ground$ a subset.
We say that $S$ satisfies the \textit{knapsack constraint}
 with the budget $b$
if $c(S):= \sum_{e \in S}c(e) \le b$. Without loss of generality,
it is sufficient to consider the unit budget case, i.e., $b=1$.

\section{PROBLEM FORMULATION} \label{sec:prob-form}
Throughout this paper, we consider the following intersection of
$l$ knapsacks and $k$-system constraints:
\begin{equation}
   \label{eq:the-constraints}
   c_{j}(S)  \le 1 \ (1 \le \forall j \le l)
   \quad
   \text{ and }
   S \in \mathcal{I}
\end{equation}
Here for $1 \le j \le l$, $c_j: \ground \rightarrow \RR_{> 0}$ is a cost and $(\ground, \mathcal{I})$ is a
$k$-system.

In this study, we consider the following sequential decision-making process for times steps $t = 1, \dots, T$.

(i) The algorithm
selects a list $S_{t} = \brl e^{(1)}_{t},
\allowbreak
\dots, \allowbreak e^{(m_t)}_{t} \brr \subseteq \ground$ satisfying
the constraints \eqref{eq:the-constraints}.

(ii) The algorithm receives noisy rewards $y_{t}^{(1)}, \dots, y_{t}^{(m_t)}$ as follows:
\begin{equation*}
   y_{t}^{(i)} = \Delta f\bl e_{t}^{(i)} \mid \setslc{S_t}{i-1} \br + \varepsilon_{t}^{(i)},
   \text{ for } i = 1, \dots, m_t,
\end{equation*}
Here $f$ is a submodular function \textit{unknown} to the algorithm,
$\setslc{S_t}{i - 1} = \brl e_{t}^{(1)}, \allowbreak \dots, \allowbreak e_{t}^{(i - 1)} \brr$ and $\varepsilon_{t}^{(i)}$ is a noise.
We regard $\setslc{S_t}{i-1}$, $e_{t}^{(i -1)}$ and $\varepsilon_{t}^{(i)}$ as random variables.
The objective of the algorithm is to maximize the sum of rewards $\sum_{t=1}^{T}f(S_t)$.

Following \parencite{yue2011linear},
we explain this problem by using the news article recommendation example.
In each round, the user scans the list of the recommended items
$S_t = \brl e_t^{(1)}, \dots, e_t^{(m_t)} \brr$ one-by-one in top-down fashion,
where $m_t$ is the cardinality of $S_t$ at round $t$.
We assume that the marginal gain $\Delta f( e_{t}^{(i)} \mid \setslc{S_t}{i-1})$
represents the new information covered by $e_t^{(i)}$ and not covered by $\setslc{S_t}{i - 1}$.
The noisy rewards $y_{t}^{(1)}, \dots, y_{t}^{(m_t)}$ are binary random variables
and the user likes $e_t^{(i)}$ with probability
$\Delta f( e_{t}^{(i)} \mid \setslc{S_t}{i-1})$.

\subsection{ASSUMPTIONS ON THE SCORE FUNCTION $f$}
Following \parencite{yue2011linear},
we assume that there exist $d$ known submodular functions
$f_1, \dots, f_d$ on $2^{\ground}$ that are linearly independent and
the objective submodular function $f$ can be written as
a linear combination $f = \sum_{i=1}^d w_i f_i$,
where the coefficients $w_1, \dots, w_d$ are non-negative and \textit{unknown} to the algorithm.
We fix a parameter $B > 0$
and assume that $\sqrt{\sum_{i=1}^d w_i^2} \le B$.
We also assume that
for some $A > 0$, the $L^2$-norm of vector
$[\Delta f_i(e \mid S)]_{i=1}^{d}$ is bounded above by $\sqrt{A}$
for any $e \in \ground$ and $S \in 2^{\ground}$.

We note that this can be generalized to an infinite dimensional case
as in \parencite{chen2017interactive}.
We discuss this setting more in detail in the supplemental material and
provide a theoretical result in this setting.

\subsection{ASSUMPTIONS ON NOISE STOCHASTIC PROCESS}
We assume that there exists $m \in \ZZ_{>0}$ such that $m_t \le m$ for all $t$
and  consider the lexicographic order on the set
$\brl (t, i) \mid t = 1, \allowbreak \dots,
\allowbreak\ 1 \le i \le m \brr$, i.e.,
$(t, i) \le (t', i')$ if and only if either $t < t'$
or $t = t'$ and $i \le i'$.
Then, we can identify the set with the set of natural numbers (as ordered sets)
and can regard $\{ \varepsilon_{t}^{(i)}\}_{t, i}$ as a sequence.
We assume that the stochastic process
$\brl \varepsilon_{t}^{(i)}\brr_{t, i}$
is \textit{conditionally $R$-sub-Gaussian}
  for a fixed constant $R \ge 0$, i.e.,
\begin{math}
   \ex{\exp \bl \xi \varepsilon_{t}^{(i)}\br \mid \mathcal{F}_{t, i} } \le
   \exp\bl \frac{\xi^{2} R^{2}}{2} \br,
\end{math}
for any $(t, i)$ and any $\xi \in \RR$.
Here, $\mathcal{F}_{t, i}$ is the $\sigma$-algebra generated by
$\left\{
   \setslc{S_u}{j} \mid (u, j) < (t, i)
\right\}
$ and
$\left\{ \varepsilon_{u}^{(j)} \mid  (u, j) < (t, i) \right\}$.
This is a standard assumption on the noise sequence
\parencite{chowdhury2017kernelized,abbasi2011improved}.
For example, if $\{ \varepsilon_{t}^{(i)}\}$
is a martingale difference sequence and $|\varepsilon_{t}^{(i)}| \le R$
or
$\{ \varepsilon_{t}^{(i)}\}$
is conditionally Gaussian with zero mean and variance $R^2$, then
the condition is satisfied \parencite{lattimore2019book}.

\subsection{APPROXIMATION REGRET}
\label{sec:alpha-regret}
As usual in the combinatorial bandit problem,
we evaluate bandit algorithms by a regret called
\textit{$\alpha$-approximation regret} (or
\textit{$\alpha$-regret} in short),
where $\alpha \in (0, 1)$. The approximation regret is necessary for meaningful evaluation.
Even if the submodular function $f$ is completely known,
it has been proved that no algorithm can achieve the optimal solution
by evaluating $f$ in polynomial time \parencite{nemhauser1978best}.

We denote by $OPT$ the optimal solution, i.e.,
\begin{math}
   OPT = \argmax_{S} f(S),
\end{math}
where $S$ runs over $2^{\ground}$ satisfying the constraint \eqref{eq:the-constraints}.
We define the $\alpha$-regret as follows:
\begin{equation*}
   \regret{\alpha}{T} = \sum_{t=1}^{T} \brl \alpha f(OPT) - f(S_t) \brr.
\end{equation*}
This definition
is slightly different from that given in \parencite{yue2011linear}
because our definition does not include noise as in \parencite{chowdhury2017kernelized}.
In either case, one can prove a similar upper bound.
For the proof in the cardinality constraint case,
we refer to Lemma 4 in the supplemental material of \parencite{yue2011linear}.

In this study, we take the same approximation ratio $\alpha = \frac{1}{(1 + \varepsilon)(k + 2l + 1)}$
as that of a fast algorithm in the offline setting \parencite[Theorem 6.1]{badanidiyuru2014fast}.
As mentioned in Section \ref{sec:related},
there exist offline algorithms that achieve better approximation ratios than above,
but they have high computational complexity.
Later, we remark that our proposed method is also ``fast''.

\section{ALGORITHM}
\label{sec:algo}
In this section, following \parencite{yue2011linear,yu2016linear}, we first define a
\textit{UCB score} of the marginal gain $\Delta f(e \mid S)$
and introduce a \textit{modified UCB score}.
With a UCB score, one can balance the exploitation and exploration tradeoff with bandit feedback.
Then, we propose a non-greedy algorithm (Algorithm \ref{algo:main}) that adaptively focuses on
the UCB score and modified UCB score.
\subsection{UCB SCORES}
For $e \in \ground$ and $S \in 2^{\ground}$, we define a column vector $x(e \mid S)$
by $\bl \Delta f_i(e \mid S) \br_{i=1}^{d} \in \RR^d$ and put $x_{t}^{(i)} =
x \bl e^{(i)}_{t} \mid \setslc{S_t}{i-1} \br$.
Here, we use the same notation as in Section \ref{sec:prob-form}.
We define $b_t, w_t \in \RR^{d}$ and $M_t \in \RR^{d\times d}$ as follows:
\begin{align*}
   b_t &:= \sum_{s=1}^{t}\sum_{i=1}^{m_s} y_s^{(i)} x_{s}^{(i)}, &\\
   M_t &:= \lambda I + \sum_{s=1}^{t}\sum_{i=1}^{m_s}x_{s}^{(i)}\otimes x_{s}^{(i)},
   &w_t := M_t^{-1}b_t,
\end{align*}
Here,
$\lambda > 0$ is a parameter of the model
and for a column vector $x \in \RR^{d}$, we denote by $x \otimes x \in \RR^{d \times d}$
the Kronecker product of $x$ and $x$.

For $e \in \ground$ and $S \in 2^{\ground}$, we define
\begin{math}
   \mu(e \mid S) := w_t \cdot x(e \mid S)
\end{math}
and
\begin{math}
\sigma(e \mid S) := \sqrt{x(e | S)^{\trn} \hspace{2pt} M_{t}^{-1} \hspace{2pt} x(e | S)}.
\end{math}
Then, we define a UCB score of the marginal gain by
\begin{equation*}
   \ucb_t(e \mid S) =
   \mu_{t-1}(e \mid S) + \beta_{t-1} \sigma_{t - 1}(e \mid S),
\end{equation*}
and a modified UCB score by $\ucb_t(e \mid S)/c(e)$.
Here,
\begin{equation*}
    \beta_t := B + R \sqrt{\ln\det \bl \lambda^{-1} M_t \br + 2 + 2\ln(1/\delta)},
\end{equation*}
and $c(e) := \sum_{j=1}^{l}c_j(e)$.
It is well-known that
$\ucb_t(e \mid \phi, S)$ is an upper confidence bound for
$\Delta f_{\phi}(e \mid S)$.
More precisely, we have the following result.
\begin{prop}
   \label{prop-ucb}
   We assume there exists $m \in \ZZ_{\ge 1}$ such that $m_t \le m$ for all $1\ \le t \le T$.
   We also assume that $\lambda \ge L$.
   Then, with probability at least $1-\delta$, the following inequality holds:
   \begin{equation*}
      \left| \mu_{t-1}(e \mid S) - \Delta f(e \mid S) \right| \le
      \beta_{t-1} \sigma_{t - 1}(e \mid S),
   \end{equation*}
   for any $t, S,$ and $e$.
\end{prop}
Proposition \ref{prop-ucb} follows from the proof of \parencite[Theorem 2]{chowdhury2017kernelized}.
We note that this theorem is a more generalized result than the statement above
(they do not assume that the objective function is linear but belongs to an RKHS).
In the linear kernel case, an equivalent result to Proposition \ref{prop-ucb}
 was proved in
\parencite{abbasi2011improved}.

We also define the UCB score for a list $S \allowbreak=\brl e^{(1)}, \dots,e^{(m)}  \brr$ by
\begin{math}
   \ucb_t(S)  \allowbreak= \mu_{t - 1}(S ) + 3 \beta_{t - 1} \sigma_{t - 1}(S ).
\end{math}
Here $\mu_{t}(S)$ and $\sigma_{t}(S)$ are defined as
\begin{math}
    \sum_{i=1}^{m} \mu_{t}(e^{(i)} \mid \setslc{S}{i - 1})
\end{math}
and
\begin{math}
    \sum_{i = 1}^{m}\sigma_t(e^{(i)} \mid \setslc{S}{i - 1}),
\end{math}
respectively.
The factor $3$ in the definition of $\ucb_t(S)$ is due to a technical reason
as clarified by the proof of Lemma \lemsupnum{} in the supplemental material.

\subsection{\algname}
\begin{algorithm}[h]
   \label{algo:sub}
   \SetKwInOut{Input}{Input}
   \SetKwInOut{Output}{Output}
   \DontPrintSemicolon
   \Input{%
      Threshold $\rho$,
      round $t$
   }
   \Output{A list $S$ satisfying the constraints \eqref{eq:the-constraints}}
   \caption{\algnamei{} (Sub-algorithm)}
   Set $S = \emptyset$, $i = 1$\;
   \While{$\mathrm{True}$}{
      $\ground_{S} = \brl e \in \ground \mid S + e
      \text{ satisfies the constraint \eqref{eq:the-constraints}}
       \brr$.
       \;
      $\ground_{S, \ge \rho} = \brl
      e \in \ground_{S} \mid
      \begin{subarray}{l}
      \ucb_{t}(e \mid S)/c(e) \ge \rho \text{ and}\\
      \ucb_{t}(e \mid \emptyset)/c(e) \ge \rho
      \end{subarray}
      \brr$.\;
      \uIf{$\ground_{S, \ge \rho} =  \emptyset$}{
         break;
      }
      $e_{i} = \argmax_{e \in \ground_{S, \ge \rho}} \ucb_t(e \mid S) $.\;
      Add $e_{i}$ to $S$. Set $i \gets i + 1$\;
   }
   Return $S$\;
\end{algorithm}

\begin{algorithm}[h]
   \label{algo:main}
   \DontPrintSemicolon
   \caption{\algname{} (Main Algorithm)}
   \SetKwInOut{Input}{Input}
   \SetKwInOut{Output}{Output}
   \Input{Parameters $B, R, \lambda, \delta, \assumedbound, \assumedbound', \varepsilon$}
   \Output{A list $S$ satisfying the constraints \eqref{eq:the-constraints}}
   \For{$t=1,\dots, T$}{
      $U = \emptyset$,
      $r = \frac{2}{k + 2l + 1}$,
      $\rho = r(1 + \varepsilon)^{-1}\assumedbound$\;
      \While{$\rho \le r\assumedbound' |\ground| $}{
         $S = \mathrm{Algorithm1}(\rho, t)$\;
         Add $S$ to $U$\;
         Set $\rho \gets (1 + \epsilon) \rho$\;
      }
   Select $S_t = \argmax_{S \in U} \ucb_{t}(S)$\;
   Receive rewards $y_t^{(1)}, \dots , y_t^{(m_t)}$\;
   }
\end{algorithm}

In this subsection, we propose a UCB-type algorithm for our problem.
We call our proposed method \algname{} and its pseudo code is outlined in
Algorithm \ref{algo:main}.
Algorithm \ref{algo:main} calls a sub-algorithm called \algnamei{} (an algorithm
 that Greedily selects elements with Modified UCB scores larger than a threshold,
 outlined in Algorithm \ref{algo:sub}).
Algorithm \ref{algo:sub} takes a threshold $\rho$ as a parameter and returns a list of elements
satisfying the constraint \ref{eq:the-constraints}.
Algorithm \ref{algo:sub} selects elements greedily
from the elements whose modified UCB scores $\ucb_t (e \mid S)/c(e)$
and $\ucb_t(e \mid \emptyset)/c(e)$ are larger or equal to the threshold $\rho$.
If the threshold $\rho$ is small, then this algorithm is almost the same
as a greedy algorithm, such as LSBGreedy \parencite{yue2011linear}.
If the threshold $\rho$ is large, then the elements with large modified UCB scores will be selected.
Thus, the threshold $\rho$ controls the importance of the standard and modified scores.
The main algorithm \ref{algo:main}
calls Algorithm \ref{algo:sub} repeatedly by changing the threshold $\rho$
and returns a list with the largest UCB score.
We prove that there exists a good list among these candidates lists.

As remarked before, Algorithm \ref{algo:main} is inspired by
submodular maximization algorithms in the
the offline setting \parencite{badanidiyuru2014fast,mirzasoleiman2016fast}.
However, we need a nontrivial modification since
the diminishing return property does not hold for $\ucb_t(e \mid S)$
unlike the marginal gain $\Delta f(e \mid S)$.
We note that $\ucb_t(e \mid S)$ can be large not only when
the estimated value of $\Delta f(e \mid S)$ is large
but also if the uncertainty in adding $e$ to $S$ is high.
Therefore, we need additional filter conditions to ensure
that $e$ is a ``good'' element.
Natural candidates for the condition are that
$\ucb_t(e \mid \setslc{S}{i})/c(e) \ge \rho$ for some indices $i$.
In Algorithm \ref{algo:main}, we require $\ucb_t(e \mid \emptyset)/c(e) \ge \rho$
in addition to $\ucb_t(e \mid S) / c(e) \ge \rho$.

In the algorithm, we introduce parameters $\assumedbound$ and $\assumedbound'$.
The parameter $\assumedbound$ (resp. $\assumedbound'$) is used for defining the initial
(resp. terminal) value of the threshold $\rho$.
In the next section, for a theoretical guarantee, we assume that
\begin{math}
   \assumedbound \le \max_{e \in \ground} f(\brl e \brr) \le \assumedbound'.
\end{math}
If the upper bound of the reward is known, then we can take
$\nu'$ as the known upper bound.
In practice, it is plausible that most users are interested in at least one item in
the entire item set $\ground$, which implies $\max_{e \in \ground} f(\brl e \brr)$
is not very small.
In addition, the number of iterations in the while loop in Algorithm \ref{algo:main} is
given by $O(\ln \bl \assumedbound' |\ground| / \assumedbound \br)$. Therefore,
taking a very small $\assumedbound$ does not increase the number of iterations as much.

\subsection{COMPUTATIONAL COMPLEXITY}
\label{sec:comp-coplx}
We discuss the computational complexity of Algorithm \ref{algo:main}
and that of existing methods.
We consider a greedy algorithm by applying LSBGreedy
to our problem; i.e., we consider a greedy algorithm
that selects the element
with the largest UCB score until the constraint is satisfied.
By abuse of terminology, we call this algorithm LSBGreedy.
Similarly, when we apply CGreedy (resp. MCSGreedy) to our problem,
we also call this algorithm CGreedy (resp. MCSGreedy).
In each round, the expected number of times to compute $\ucb_{t}(e \mid S)$
in Algorithm \ref{algo:main} is given by
$O(m |\ground| \ln (\nu' |\ground|/\nu)/\ln(1 + \epsilon))$,
while that of LSBGreedy is given by $O(m |\ground|)$.
The computational complexity of MCSGreedy and CGreedy is given as $O(|\ground|^3)$ and $O(m|\ground|)$
respectively. Therefore,
ignoring unimportant parameters,
our algorithms incur an additional factor $\ln |\ground|/ \ln (1 + \varepsilon)$
compared to that of LSBGreedy and CGreedy.

\section{MAIN RESULTS}
\label{sec:main}
The main challenge of this paper is to provide a strong theoretical result for
\algname{}.
In this section, under the assumptions stated as in the previous section,
we provide an upper bound for the approximation regret of \algname{}
and give a sketch of the proof.
We also show that existing greedy methods incur linear approximation regret in the worst
case for our problem.
\subsection{STATEMENT OF THE MAIN RESULTS}
\begin{thm}
   \label{thm:main}
   Let the notation and assumptions be as previously mentioned.
   We also assume that $\lambda \ge m$.
   We let $\alpha = \frac{1}{(1+\varepsilon) \bl k + 2l + 1\br}$.
   Then, with probability at least $1-\delta$,
   the proposed algorithm achieves the following $\alpha$-regret bound:
   \begin{equation*}
      \regret{\alpha}{T}
      \le 8 A\beta_{T} \sqrt{2 mT \ln \det\bl \lambda^{-1}M_T \br}.
   \end{equation*}
   In particular, ignoring $A, B, R$, we have
   $\regret{\alpha}{T} = O(d\sqrt{mT} \ln \frac{mT}{\delta})$ with probability at least $1 - \delta$.
\end{thm}
\begin{rem}
   \begin{enumerate}
      \item In the published version of UAI 2020, the assumption ``$\lambda \ge m$'' was missing.
      However, we need this assumption as in the previous work \parencite[Lemma 5]{yue2011linear}.
      \item
      There is a tradeoff between the approximation ratio and computational complexity.
      As discussed in Section \ref{sec:comp-coplx},
      the computational complexity of the algorithm is given as
      $O(m |\ground| \ln (|\ground|)/\ln(1 + \epsilon))$ in each round, while
      the approximation of the algorithm is given as $\frac{1}{(1+\varepsilon) \bl k + 2l + 1\br}$.
      \item We assume the score function $f$ is a linear combination of known
      submodular functions.
      We can relax the assumption to the case when the function $(e, S) \rightarrow\Delta f(e | S)$
      belongs to an RKHS and has a bounded norm in the space as in \parencite{chen2017interactive}.
      We discuss this setting more in detail and provide a generalized result
       in the supplemental material.
   \end{enumerate}
\end{rem}

In the setting of \parencite{yue2011linear,yu2016linear},
greedy methods have good theoretical properties.
However, we show that for any $\alpha > 0$, these greedy methods incur linear $\alpha$-regret in the worst case
for our problem.
We denote by $\regretmcs(T)$ and $\regretlsb(T)$
the $\alpha$-regret of MCSGreedy and that of LSBGreedy, respectively.
Then the following proposition holds.
\begin{prop}
   \label{prop:greedy-linear}
   For any $\alpha > 0$,
   there exists cost $c_1$, $k$-system $\mathcal{I}$, a submodular function $f$,
   $T_0 > 0$ and a constant $C > 0$ such that
   with probability at least $1-\delta$,
   \begin{equation*}
      \regretmcs(T) > C T,
   \end{equation*}
   for any $T > T_0$.
   Moreover, the same statement holds for $\regretlsb(T)$.
\end{prop}
We provide the proof in the supplemental material.

\subsection{SKETCH OF THE PROOF OF THEOREM \ref{thm:main}}
We provide a sketch of the proof of Theorem \ref{thm:main}
and provide a detailed and generalized proof in the supplemental material.
Throughout the proof, we fix the event $\mathcal{F}$ on which the inequality in
Proposition \ref{prop-ucb} holds.

We evaluate the solution $S_t$ by \algname{} in each round $t$.
The following is a key result for our proof of Theorem \ref{thm:main}.
\begin{prop}
   \label{prop:sub-greedy}
   Let $C \subseteq \ground$ be any set satisfying
   the constraint \eqref{eq:the-constraints}.
   Let $S$ be a set returned by \algnamei{} at time step $t$.
   Then, on the event $\mathcal{F}$, we have
   \begin{equation*}
      f(S) + 2 \beta_{t-1}\sigma_{t-1}(S) \ge \min \brl
      \frac{\rho}{2},
      \frac{1}{k + 1} f(S \cup C) - \frac{l \rho}{k + 1}
      \brr.
   \end{equation*}
\end{prop}
\begin{proof}[sketch of proof]
   This can be proved in a similar way to the proof of
   \parencite[Theorem 6.1]{badanidiyuru2014fast} or
   \parencite[Theorem 5.1]{mirzasoleiman2016fast}.
   However, because of uncertainty and lack of diminishing property of the UCB score,
   we need further analysis.
   We divide the proof into two cases.

   \textbf{Case One}.
   This is the case when \algnamei{} terminates because
   there exists an element $e$
   such that $\ucb_t(e \mid S) \ge \rho c(e)$
   and $\ucb_t(e \mid \emptyset) \ge \rho c(e)$,
   but any element $e$ satisfying
   $\ucb_t(e \mid S),\ucb_t(e \mid \emptyset) \ge \rho c(e)$
   does not satisfy the knapsack constraints, i.e., $c_j(S + e) > 1$ for some $1 \le j \le l$.
   We fix an element $e'$ satisfying $\ucb_t(e' \mid S), \ucb_t(e' \mid \emptyset) \ge \rho c(e')$.
   Because any element of $S$ has enough modified UCB score,
   by Proposition \ref{prop-ucb}, we have
   \begin{math}
      f (S) + 2\beta_{t-1}\sigma_{t-1}(S)
      \ge \rho c(S).
   \end{math}
   By the definition of $e'$, we also have $\ucb_t(e' \mid \emptyset) \ge \rho c(e')$.
   Because $f (S) + 2\beta_{t-1}\sigma_{t-1}(S) \ge \ucb_t(e' \mid \emptyset) \ge \rho c(e')$ and
   $S + e'$ does not satisfy the knapsack constraint,
   we have
   \begin{math}
      f (S) + 2\beta_{t-1}\sigma_{t-1}(S)
      \ge \frac{\rho}{2} c(S + e') \ge \rho / 2.
   \end{math}

   \textbf{Case Two}.
   This is the case when \algnamei{} terminates because
   for any element $e$ satisfying $\ucb_t(e \mid S),\ucb_t(e \mid \emptyset) \ge \rho c(e)$,
   $e$ satisfies the knapsack constraints
   but $S + e$ does not satisfy the $k$-system constraint.
   We note that this case includes the case
   when there does not exist an element $e$ satisfying
   $\ucb_t(e \mid S),\ucb_t(e \mid \emptyset) \ge \rho c(e)$.

   We define a set $C_{< \rho}$ as
   \begin{align*}
   \brl e \in C \mid \exists i \text{ such that }
   \ucb_t(e \mid \setslc{S}{i}) < \rho c(e)
   \brr,
   \end{align*}
   and $C_{\ge \rho } = C \setminus C_{< \rho}$.
   Let $e \in C_{< \rho}$.
   Then on the event $\propevent$, by Proposition \ref{prop-ucb} and submodularity,
   we have
   \begin{equation}
      \label{eq:crhorho}
      \Delta f(C_{< \rho} \mid S) \le
      \sum_{e \in C_{< \rho}}  \Delta f(e \mid S)
      \le \sum_{e \in C_{< \rho}}  \rho c(e)
      \le l \rho.
   \end{equation}

   Next, we consider $C_{\ge \rho}$.
   Running the greedy algorithm (with respect to the UCB score) on $S \cup C_{\ge \rho}$
   under only the $k$-system constraint,
   we obtain $S$ by the assumption of this case.
   Then, it can be proved that
   \begin{math}
      f(S) + 2\beta_{t-1}\sigma_{t-1}(S) \ge
      \frac{1}{k + 1}f(S \cup C_{\ge \rho}).
   \end{math}
   We note that this is a variant of the result proved in \parencite[Appendix B]{calinescu2011maximizing}.
   By this inequality, inequality \eqref{eq:crhorho}, and submodularity, we can derive the desired result.
\end{proof}
Using Proposition \ref{prop:sub-greedy}, we can bound the approximation
regret above by the sum of uncertainty $\beta_{t-1}\sigma_{t-1}(S_t)$.
Because the algorithm selects $S_t$ and obtain feedbacks for $S_t$,
the sum of uncertainty can be bounded above by a sub-linear function of $T$.

\section{EXPERIMENTAL ANALYSIS}
\newcommand{\figureheight}{2.8cm}
\begin{figure*}[ht]
   \begin{center}
      \includegraphics[width=\linewidth,height=\figureheight]{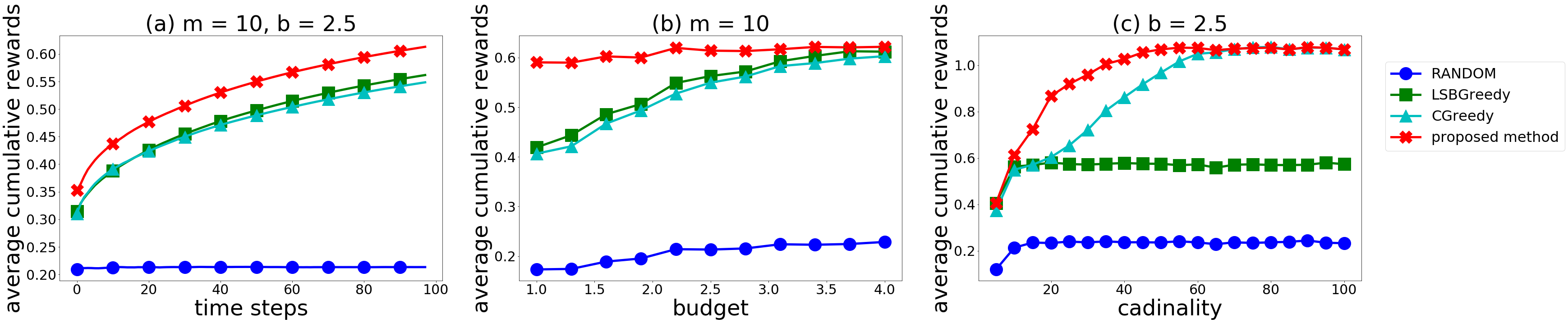}
      \caption{Cumulative average rewards on the synthetic news article recommendation dataset}
      \label{fig:art}
   \end{center}
\end{figure*}


\begin{figure*}[h]
   \begin{center}
      \includegraphics[width=\linewidth,height=\figureheight]{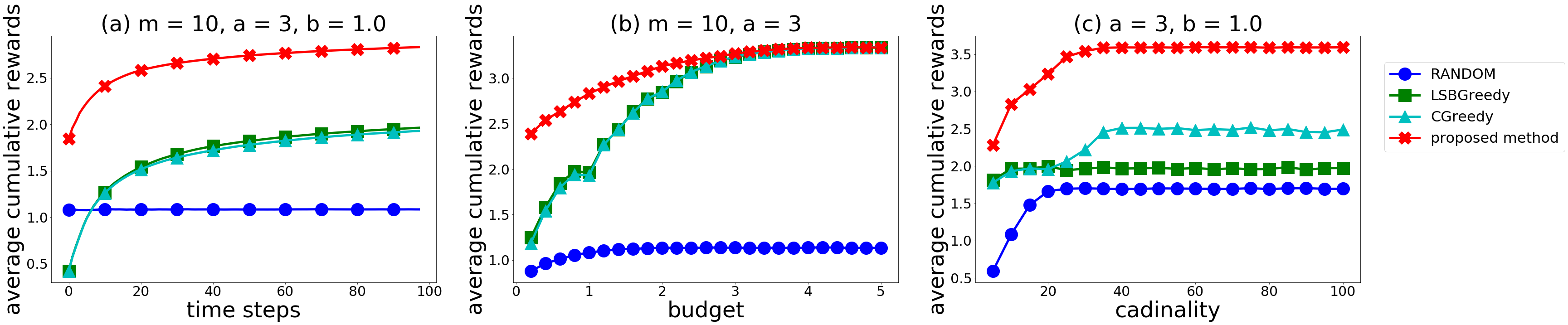}
      \caption{Cumulative average rewards on the MovieLens dataset}
      \label{fig:graph2}
   \end{center}
\end{figure*}

\begin{figure*}[h]
   \begin{center}
      \includegraphics[width=\linewidth,height=\figureheight]{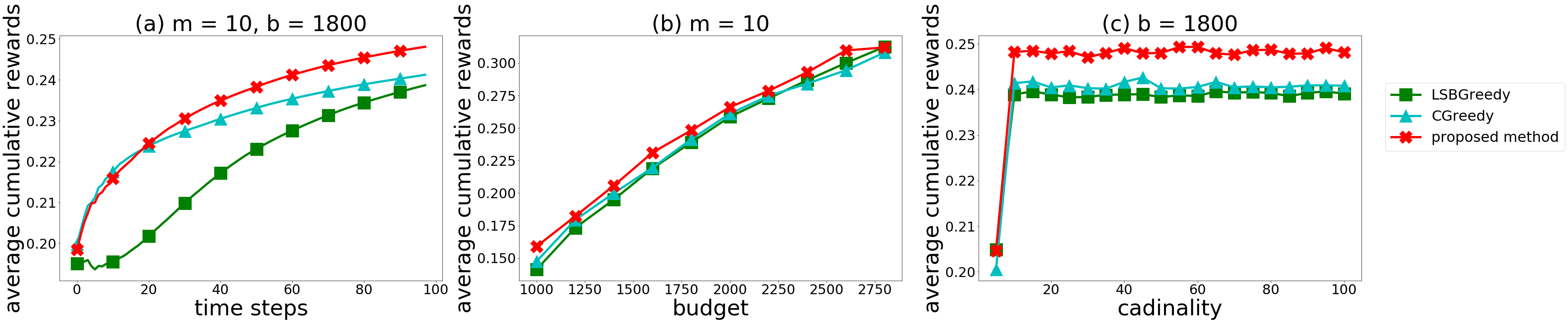}
      \caption{Cumulative average rewards on the Million Song Dataset}
      \label{fig:msd}
   \end{center}
\end{figure*}

\label{sec:exp}
In this section,
we empirically evaluate our methods by
a synthetic dataset that simulates an environment for news article recommendation
and two real-world datasets
(MovieLens100K \parencite{movielens100k} and the Million Song Dataset \parencite{Bertin-Mahieux2011}).

We compare our proposed algorithm to the following baselines:
\begin{itemize}
   \item RANDOM. In each round, this algorithm selects elements uniform randomly until no element satisfies the constraints.
   \item LSBGreedy. This was proposed in \parencite{yue2011linear}
   to solve the submodular bandit problem under a cardinality constraint.
   In the linear kernel case,
   SM-UCB \parencite{chen2017interactive} is equivalent to LSBGreedy.
   \item CGreedy. This is an algorithm for a submodular bandit problem under a knapsack constraint
   and was proposed in \parencite{yu2016linear}.
   They also proposed an algorithm called MCSGreedy. However because MCSGreedy is computationally expensive
   (in each round it calls functions $f_1, \dots, f_d$ for $O(|\ground|^3)$ times) and
   their experimental results show that both algorithms have
   a similar empirical performance, we do not add MCSGreedy to the baselines.
\end{itemize}
In Proposition \ref{prop:greedy-linear}, we showed that these greedy algorithms
incur linear approximation regret in the worst case.
However, even without theoretical guarantee,
it is empirically known that a greedy algorithm achieve a good experimental performance.
In this section, we demonstrate that our algorithm outperforms these greedy algorithms under
various combinations of constraints.
As a special case, such constraints include the case when there is a sufficiently large
budget for knapsack constraints and
the case when the $k$-system constraint is sufficiently mild.
The greedy algorithms are algorithms for such cases.
We also show that our proposed method performs no worse than the baselines
even in these cases.

As in the preceding work \parencite{yue2011linear}, we assume the score function $f$
is a linear combination of known probabilistic coverage functions.
We assume there exists a set of topics (or genres) $\mathcal{G}$ with $|\mathcal{G}| = d$ and
for each item $e \in \ground$,
there is a feature vector
$x(e) := (P_g(e))_{g \in \mathcal{G}} \in \RR^d$ that represents the information coverage
on different genres.
For each genre $g$, we define the probabilistic coverage function $f_g(S)$ by
$1 - \prod_{e \in S}(1 - P_g(e))$ and we assume $f = \sum_{i}w_i f_i$
with unknown linear coefficients $w_i$.
The vector $w := [w_1, \dots, w_d]$ represents user preference on genres.
We assume that the noisy rewards $y_t^{(i)}$ are sampled by
 $y_t^{(i)} \sim \mathrm{Ber} \bl \Delta f(e_t^{(i)} \mid \setslc{S_t}{i-1}) \br$.
Below, we define these feature vectors $x(e)$, $w$, and constraints explicitly.
We note that in the experiments, we use an un-normalized knapsack constraint $c(S) \le b$.
In the following experiments, using 100 users (100 vectors $w$),
we compute cumulative average rewards for each algorithm.
When taking the average, we repeated this experiment 10 times for each user.

\subsection{NEWS ARTICLE RECOMMENDATION}
In this synthetic dataset, we assume $d = 15$ and $|\ground| = 1000$.
We define $x(e)$ and costs for a knapsack constraint in a similar manner in \parencite{yu2016linear}.
We sample each entry of $x(e)$ from two types of uniform distributions.
We assume that for each item $e$,
the number of genres that have high information coverage is limited to two.
More precisely, we randomly select two indices of $x(e)$ and sample entries from
$U(0.5, 0.8)$ and sample other entries from $U(0.0, 0.01)$.
We generate 100 user preference vectors $w$ in a similar way to $x(e)$.
We also sample the costs of items uniform randomly from $U(0.0, 1.0)$.
In this dataset, we consider the intersection of a cardinality constraint and a knapsack constraint.
The result is shown in Figure \ref{fig:art}.

\subsection{MOVIE RECOMMENDATION}
We perform a similar experiment in \parencite{mirzasoleiman2016fast} but with a semi-bandit feedback.
In MovieLens100K, there are 943 users and 1682 movies.
We take $\ground$ as the set of 1682 movies in the dataset.
There are $d = 18$ genres in this dataset.
First, we fill the ratings for all the user-item pairs using matrix factorization \parencite{koren2009matrix}
and we normalized the ratings $r$ so that $r \in [0, 1]$.
For each movie $e \in \ground$, we denote by $r_e \in [0, 1]$ the mean of the
ratings of the movie for all users.
We define $P(g \mid e) = r_e / |\mathcal{G}_e|$ if $g \in \mathcal{G}_e$,
otherwise we define $P(g \mid e) = 0$.
We normalize $P(g \mid e)$ as previously mentioned, because
if $w_i = 1$ for all $i$, then we have $P(\{e \}) = r_e$.

We define a similar knapsack, cardinality, and matroid constraints to
those of \parencite{mirzasoleiman2016fast}.
For $e \in \ground$, the cost $c(e)$ is defined as $c(e) = F_{\mathrm{Beta}(10, 2)}(r_e)$, where
$F_{\mathrm{Beta}(10, 2)}$ is the cumulative distribution function of the $\mathrm{Beta}(10, 2)$.
For a budget $b \in \RR_{> 0}$, we consider a knapsack constraint $c(S) \le b$.
The beta distribution lets us differentiate the highly rated movies from those with lower ratings
\parencite{mirzasoleiman2016fast}.
We generate 100 user preference vectors $w$ in a similar way to
the news article recommendation example.
In this dataset, we consider the following constraints on genres
in addition to the knapsack $c(S) \le b$ and cardinality $|S| \le m$ constraints,
There are $k$ genres in MovieLens100K, where $k = d = 18$.
For each genre $g$, we fix a non-negative integer $a$ and consider the constraint
\begin{math}
   |\brl e \in S \mid e \text{ has genre } g \brr| \le a
\end{math}
for $S \subseteq \ground$.
This can be regarded as a partition matroid constraint.
Therefore, the intersection of the constraints for all genres
is a $k$-system constraint.
One can prove that the intersection of this $k$-system constraint and
a cardinality constraint is also a $k$-system constraint.
The results are displayed in Figure \ref{fig:graph2}
in the case of the matroid limit $a = 3$.

\subsection{MUSIC RECOMMENDATION}
From the Million Song Dataset, we select 1000 most popular songs and 30 most popular genres.
Thus, we have $|\ground| = 1000$ and $d = 30$.
For active 100 users, we compute $P_g(e)$ and user preference vector $w$
in almost the same way as $\overline{w}(e, g)$ and
$\theta^{*}$ in \parencite{hiranandanicascading} respectively.
They assume that a user likes a song $e$ if the user listened to the song at least five times,
however, we assume that a user likes the song if the user listened to the song at least two times.
We consider the intersection of a cardinality and a knapsack constraint
$c(S) \le b$.
We define a cost $c$ for the knapsack constraint by the length (in seconds) of the song in the dataset.
The costs represent the length of time spent by users before they decide to listen to the song
and we assume that it is proportional to the length of the song
\footnote{We can also assume that users listen to the song and give feedbacks later.}.
The results are displayed in Figure \ref{fig:msd}.
We do not show the performance of RANDOM in the figure since it achieves only very low rewards.

\subsection{RESULTS}
In Figures \ref{fig:art}a, \ref{fig:graph2}a, \ref{fig:msd}a,
we plot the cumulative average rewards for each algorithm up to time step $T = 100$.
In Figures \ref{fig:art}b, \ref{fig:graph2}b, and, \ref{fig:msd}b
(resp. \ref{fig:art}c, \ref{fig:graph2}c, and, \ref{fig:msd}c),
we show the cumulative average rewards at the final round by changing the budget $b$
(resp. by changing the cardinality limit $m$) and
fixing the cardinality limit $m$ (resp. fixing the budget $b$).
These results shows that overall our proposed method outperforms the baselines.
We note that Figure \ref{fig:msd} shows different tendency
as compared to other datasets since
popular items in the Million Song Dataset have high information coverage for multiple genres
and about 47 $\%$ of the items have low information coverage (less than 0.01) for all genres.
Figures \ref{fig:art}b, \ref{fig:graph2}b, and \ref{fig:msd}b
also show the results for the case when the budget is sufficiently large.
This is the case when LSBGreedy performs well and our experimental results show that even in this case,
our method have comparable performance to greedy algorithms.
Moreover, Figures \ref{fig:art}c, \ref{fig:graph2}c, and \ref{fig:msd}c
also show the results in the case when the cardinality constraints are sufficiently mild.
In this case, CGreedy performs well since the constraints are almost same as a knapsack constraint.
The experimental results show that our method
tends to have better performance than that of CGreedy even in this case.

\section{CONCLUSIONS}
In this study,
motivated by diversified retrieval considering cost of items,
we introduced the submodular bandit problem
under the intersection of a $k$-system and knapsack constraints.
Then, we proposed a non-greedy algorithm to solve the problem
and provide a strong theoretical guarantee.
We demonstrated our proposed method outperforms
the greedy baselines using synthetic and two real-world datasets.

A possible generalization of this work is a generalization to
the full bandit setting. In this setting,
a leaner observes only a value $f(S_t) + \epsilon$ in each round.
Since it needs much work to derive a theoretical guarantee,
we leave this setting for future work.


\section*{APPENDIX}
In this appendix, we generalize the reward model considered in the main article
to the kernelized setting as in \parencite{chen2017interactive}.
We also provide parameters used in the experiments.
\section{PROBLEM FORMULATION UNDER A GENERALIZED REWARD MODEL}
\label{sec:prob-form-gen}
In this appendix, we consider the same reward model as in \parencite{chen2017interactive},
but subject to the intersection of knapsacks and $k$-system constraint as in the main article.
SM-UCB \parencite{chen2017interactive} is based on
CGP-UCB \parencite{krause2011contextual} or GP-UCB \parencite{srinivas2010gaussian}.
However, recently, \parencite{chowdhury2017kernelized} improved
the assumption and the regret analysis of
GP-UCB \parencite{srinivas2010gaussian}.
We follow setting of \parencite{chowdhury2017kernelized}.

Let $\Phi$ be a compact subset of some Euclidean space, which represents the space of
contexts.
We consider the following sequential decision making process for times steps $t = 1, \dots, T$.
\begin{enumerate}
   \item The algorithm observes context $\phi_t \in \Phi$ and
   selects a list $S_{t} = \brl e^{(1)}_{t}, \dots, e^{(m_t)}_{t} \brr \subseteq \ground$ satisfying
   the constraints.
   \item The algorithm receives noisy rewards $y_{t}^{(1)}, \dots, y_{t}^{(m_t)}$ as follows:
   \begin{equation*}
      y_{t}^{(i)} = \Delta f_{\phi_t}\bl e_{t}^{(i)} \mid \setslc{S_t}{i-1} \br + \varepsilon_{t}^{(i)},
   \end{equation*}
   for $i = 1, \dots, m_t$.
   Here
   $f_{\phi_t}$ is a non-negative, monotone
   submodular function unknown to the algorithm,
   $\setslc{S_t}{i - 1} = \brl e_{t}^{(1)}, \dots, e_{t}^{(i - 1)} \brr$ and $\varepsilon_{t}^{(i)}$ is a noise.
\end{enumerate}
Here, we regard $\phi_t, \setslc{S_t}{i-1}$, $e_{t}^{(i -1)}$ and $\varepsilon_{t}^{(i)}$ as random variables.

\subsection{ASSUMPTIONS REGARDING THE SCORE FUNCTION $f_{\phi}$}
The linear model considered in the main article
can be generalized to an infinite dimensional case \parencite{chen2017interactive}.
We let $D := \Phi \times 2^{\ground} \times \ground$
and define $\chi: D \rightarrow \RR$
by $\chi((\phi, S, e)) = \Delta f_\phi (e \mid S)$.
We assume that there exists an RKHS (reproducing kernel Hilbert space)
$\mathcal{H}$ on $D$ with a positive definite (or linear) kernel $\kappa$
and $\chi$ belongs to $\mathcal{H}$ and the norm $|\chi|_\mathcal{H}$ is bounded by $B > 0$.
We assume that $\kappa$ is continuos on $D \times D$.
We also assume that $\kappa(x, x) \le 1$ for any $x \in D$.
If $\kappa(x, x) \le c$, then our $\alpha$-regret would increase by a factor of $\sqrt{c}$.
\subsection{ASSUMPTIONS REGARDING NOISE STOCHASTIC PROCESS}
As for noises, we consider the same assumption as in the main article.

\section{DEFINITION OF UCB SCORES}
In this section, following \parencite{chen2017interactive,chowdhury2017kernelized},
we generalize UCB scores defined in the main article.

We let $x_t^{(i)} = (\phi_{t}, \allowbreak \setslc{S_t}{i - 1},\allowbreak e_{t}^{(i)}) \in D$ and
$\mathbf{x}_{(1:t)} = (x_1^{(1)},\allowbreak \dots,\allowbreak x_1^{(m_1)},\allowbreak
\dots,\allowbreak x_t^{(1)},\allowbreak \dots x_t^{(m_t)}) \in D^{M}$,
where $M = \sum_{s=1}^{t}m_s$.
We also define $\mathbf{y}_{(1:t)} \in \RR^{M}$ as
$(y_1^{(1)},\allowbreak \dots,\allowbreak y_1^{(m_1)},\allowbreak \dots,\allowbreak y_t^{(1)},
\allowbreak \dots y_t^{(m_t)})$.
For a sequence $\xi = (\xi_1, \dots, \xi_s) \in D^{s}$, we define
$K(\xi) = \bl \kappa(\xi_i, \xi_j) \br_{1\le i, j \le s} \in \RR^{s \times s}$ and
$\kappa(\xi, x) = [\kappa(\xi_1, x), \dots, \kappa(\xi_s, x)]^{\trn} \in \RR^{s}$.
We also let $\kappa_t(x) = \kappa(\mathbf{x}_{(1:t)}, x) \in \RR^{M}$ and
$K_t = K(\mathbf{x}_{(1:t)}) \in \RR^{M \times M}$.
Then, we define $\mu_{t}(x) \in \RR $ and $ \sigma_{t}^{2}(x) \in \RR_{> 0}$ as follows:
\begin{align*}
   \mu_{t}(x) &= \kappa_{t} (x)^{\trn} \bl K_t + \lambda I \br^{-1}  \mathbf{y}_{1:t},\\
   \kappa_t(x, x') &= \kappa(x, x') - \kappa_t(x)^{\trn} \bl K_t + \lambda I \br^{-1} \kappa_t(x'),\\
   \sigma_t^2(x) &= \kappa_t(x, x).
\end{align*}
Here, $\lambda \in \RR_{>0}$ is a parameter of the model.
If $x = (\phi, S, e)$, we also write $\mu_{t}(e \mid \phi,S) = \mu_{t}(x)$
and $\sigma_{t}(e \mid \phi, S) = \sigma_{t}(x)$.

We define a UCB score of the marginal gain by
\begin{equation*}
   \ucb_t(e \mid \phi, S) =
   \mu_{t-1}(e \mid \phi, S) + \beta_{t-1} \sigma_{t - 1}(e \mid \phi, S),
\end{equation*}
and a modified UCB score by $\ucb_t(e \mid \phi, S)/c(e)$.
Here, $\beta_{t}$ is defined as
\begin{math}
    B + R \sqrt{2 \bl \gamma_{M_t} + 1 + \ln(1/\delta) \br}
\end{math}
and $M_t = \sum_{s=1}^{t} m_s$.
Here $\gamma_s$ is the maximum information gain
\parencite{srinivas2010gaussian,chowdhury2017kernelized} after observing $s$ rewards.
We refer to \parencite{chowdhury2017kernelized} for the definition.

The following proposition is a generalization of Proposition \ref{prop-ucb}
and is a
direct consequence of
(the proof of) \parencite[Theorem 2]{chowdhury2017kernelized}.
\begin{prop}
   \label{prop-ucb-gen}
   We assume there exists $m \in \ZZ_{\ge 1}$ such that $m_t \le m$ for all $1\ \le t \le T$.
   Then, with probability at least $1-\delta$, the following inequality holds:
   \begin{equation*}
      \left| \mu_{t-1}(e \mid \phi, S) - \Delta f_{\phi}(e \mid S)
      \right| \le
      \beta_{t-1} \sigma_{t - 1}(e \mid \phi, S),
   \end{equation*}
   for any $t, \phi, S,$ and $e$.
\end{prop}

\section{STATEMENT OF THE MAIN THEOREM}
With generalized UCB scores, we consider the same algorithm in the main article.
Then, we provide a statement for the generalized version of Theorem \ref{thm:main}.

\begin{thm}
   \label{thm:main-gen}
   Let the notation and assumptions be as previously mentioned.
   We also assume that $\lambda \ge m$.
   We let $\alpha = \frac{1}{(1+\varepsilon) \bl k + 2l + 1\br}$, and define $\alpha$-regret as
   \begin{equation*}
      \regret{\alpha}{T} = \sum_{t = 1}^{T}\alpha f_{\phi_t}(OPT_t) - f_{\phi_t}(S_t),
   \end{equation*}
   where $OPT_t$ is a feasible optimal solution at round $t$.
   Then, with probability at least $1-\delta$,
   the proposed algorithm achieves the following $\alpha$-regret bound:
   \begin{equation*}
      \regret{\alpha}{T}
      \le 8 \beta_{T} \sqrt{ mT \gamma_{mT}}.
   \end{equation*}
   In particular,
   with at least probability $1 - \delta$, the $\alpha$-regret
   $\regret{\alpha}{T}$ is given as
   \begin{equation*}
      O\bl B\sqrt{T' \gamma_{T'}}  +  R\sqrt{T'\gamma_{T'} \bl
         \gamma_{T'} + 1 + \ln(1/\delta)
      \br} \br,
   \end{equation*}
   where $T' = mT$.
\end{thm}
\begin{rem}
   \begin{enumerate}
     \item The maximum information gain $\gamma_T$ is $O(d \ln T)$ and
     $O((\ln T)^{d + 1})$ if the kernel is a $d$-dimensional linear and
     Squared Exponential kernel, respectively \parencite{srinivas2010gaussian}.
     They also showed that a similar result for the Mat\`{e}rn kernel.
     Thus, if the kernel is a $d$-dimensional kernel, up to a polylogarithmic factor,
     we obtain Theorem 1 in the main article as a corollary.
     \item  In the main article, we assume that the norm of vector
     $[\Delta f_i(e \mid S)]_{i=1}^{d}$ is bounded above by $A$.
     Moreover, the factor $A$ appears in the regret bound in Theorem \ref{thm:main}.
     However, since we normalize the kernel so that $\kappa(x, x) \le 1$,
     a factor corresponding to $A$ does not appear in Theorem \ref{thm:main-gen}.
   \end{enumerate}
\end{rem}

\section{PROOF OF THE MAIN THEOREM}
Throughout the proof,
we assume that assumptions of Proposition \ref{prop-ucb-gen} hold
and fix the event $\mathcal{F}$ on which the inequality in Proposition \ref{prop-ucb-gen} holds.

\subsection{GREEDY ALGORITHM UNDER A $k$-SYSTEM CONSTRAINT}
In this subsection, we fix time step $t$ and context $\phi$
and consider the greedy algorithm under only $k$-system constraint as shown in
Algorithm \ref{algo:greedy}.
Here, we drop $\phi$ from notation.
We denote by $(\ground, \mathcal{I})$ the $k$-system.

\begin{algorithm}
   \label{algo:greedy}
   \SetKwInOut{Input}{Input}
   \SetKwInOut{Output}{Output}
   \DontPrintSemicolon
   \caption{GREEDY UCB}
   $S = \emptyset$\;
   \While{True}{
      $\ground' = \brl e \in \ground \mid S + e \in \mathcal{I} \brr$\;
      \uIf{$\ground' \neq \emptyset$}{
         $e = \argmax_{e \in \ground'} \ucb_t(e \mid \phi, S)$\;
         Add $e$ to $S$\;
      }\Else{
         break\;
      }
   }
   Return $S$\;
\end{algorithm}

\begin{prop}
   \label{prop:greedy-supp}
   Let $S$ be a set returned by Algorithm \ref{algo:greedy}.
   Then for any feasible set $C$, on the event $\propevent$,
   the following inequality holds:
   \begin{equation*}
      f(S) + \frac{2k\beta_{t - 1}}{k + 1}
      \sigma_{t-1}(S) \ge \frac{1}{k + 1} f(S \cup C).
   \end{equation*}
\end{prop}
\begin{proof}
   This can be proved in a similar way to the proof in \parencite[Appendix B]{calinescu2011maximizing}.
   We note the proof given in Appendix B works even if we replace the optimal solution $O$ by any feasible set $C$.
   We write $S = \brl e_1, \dots, e_m \brr$, where $e_1, \dots, e_m$ are added by
   Algorithm \ref{algo:greedy} with this order.
   We construct a partition $C_1, \dots, C_m$ of $C$ in
    the same way to the construction of $O_i$ in \parencite{calinescu2011maximizing}.
   Then $|C_i| \le k$ for any $i$ and $\setslc{S}{i - 1} + e$ is feasible for any $e \in C_i$.
   By the choice of the greedy algorithm and Proposition \ref{prop-ucb-gen},
   with probability at
   least $1 - \delta$, we have
   \begin{align*}
      &\Delta f(e \mid \setslc{S}{i - 1})\\
      &\le \mu_{t-1} (e \mid \setslc{S}{i-1}) +
      \beta_{t-1}\sigma_{t-1}(e \mid \setslc{S}{i-1})\\
      &\le \mu_{t-1} (e_i \mid \setslc{S}{i-1}) +
      \beta_{t-1}\sigma_{t-1}(e_i \mid \setslc{S}{i-1})\\
      & \le
      \Delta f(e_i \mid \setslc{S}{i - 1}) +
      2\beta_{t-1}\sigma_{t-1}(e_i \mid \setslc{S}{i-1}),
   \end{align*}
   for any $i$ and any $e \in C_i$.
   Here we use Proposition \ref{prop-ucb-gen} in the first and third inequality.
   The second inequality follows from the choice of the greedy algorithm and
   the fact that $\setslc{S}{i-1} + e$ is feasible.
   Noting that $|C_i| \le k$ and taking the sum of both sides for
   $e \in C_i$, we have
   \begin{align*}
      &k \bl
      \Delta f(e_i \mid \setslc{S}{i-1})  + 2\beta_{t-1}\sigma_{t-1}(e_i \mid \setslc{S}{i-1})
      \br \\
      & \ge
      \sum_{e \in C_i} \Delta f(e \mid \setslc{S}{i-1})\\
      &\ge \Delta f(C_i \mid \setslc{S}{i - 1}) \ge \Delta f(C_i \mid S).
   \end{align*}
   Here for subsets $A, B \in \ground$, we define $\Delta f(A \mid B) = f(A\cup B) - f(B)$
   and in the second and third inequality, we use the submodularity.
   By taking the sum of both sides, we have
   \begin{align*}
      &k \bl
      f(S) - f(\emptyset)
       + 2\beta_{t-1}\sum_{i=1}^{m}\sigma_{t-1}(e_i \mid \setslc{S}{i-1})
      \br \\
      &\ge \sum_{i=1}^{m}\Delta f(C_i \mid S)
      \ge \Delta f(C \mid S) = f(C \cup S) - f(S).
   \end{align*}
   Here the second inequality follows from submodularity of $f$.
   Thus by non-negativity of $f$, we have our assertion.
\end{proof}

\subsection{SOLUTIONS OF \algname{} AT EACH ROUND}
In this subsection, we assume that assumptions of Theorem \ref{thm:main-gen} are satisfied.
The objective in this subsection is to provide a lower bound of the score of
the set returned by \algname{} at each round.
In this subsection, we also fix time step $t$.

In the next proposition, we consider sets returned by \algnamei{}.
\begin{prop}
   \label{prop:sub-greedy-supp}
   Let $C \subseteq \ground$ be any set satisfying
   knapsack and
   $k$-system constraints.
   Let $S$ be a set returned by \algnamei{}.
   Then with probability at least $1-\delta$, we have
   \begin{equation*}
      f(S) + 2 \beta_{t-1}\sigma_{t-1}(S)
      \ge
      \min \brl
      \frac{\rho}{2},
      \frac{1}{k + 1} f(S \cup C) - \frac{l \rho}{k + 1}
      \brr.
   \end{equation*}
\end{prop}
\begin{proof}
   This can be proved in a similar way to the proof of
   \parencite[Theorem 6.1]{badanidiyuru2014fast} or
   \parencite[Theorem 5.1]{mirzasoleiman2016fast}.
   However, because of uncertainty, we need further analysis.
   We divide the proof into two cases.

   \textbf{Case One}.
   This is the case when \algnamei{} terminates because
   there exists an element $e$
   such that $\ucb_t(e \mid S) \ge \rho c(e)$
   and $\ucb_t(e \mid \emptyset) \ge \rho c(e)$,
   but any element $e$ satisfying
   $\ucb_t(e \mid S),\ucb_t(e \mid \emptyset) \ge \rho c(e)$
   does not satisfy the knapsack constraints, i.e., $c_j(S + e) > 1$ for some $1 \le j \le l$.
   We fix an element $e'$ satisfying $\ucb_t(e' \mid S), \ucb_t(e' \mid \emptyset) \ge \rho c(e')$.
   We write $S = \brl e_1, \dots, e_m \brr$.
   Then by assumption, for any $i$, we have $\ucb_t(e_i \mid \setslc{S}{i-1})\ge \rho c(e_i)$.
   By Proposition \ref{prop-ucb-gen},
   on the event $\propevent$,
   we have $ \Delta f(e_i \mid \setslc{S}{i-1}) + 2\beta_{t-1}
   \sigma_{t-1}(e_i \mid \setslc{S}{i-1}) \ge \rho c(e_i)$ for any $t, i$.
   By summing the both sides, we obtain
   \begin{equation*}
      f (S) + 2\beta_{t-1}\sigma_{t-1}(S) \ge \rho c(S).
   \end{equation*}
   On the event $\propevent$, we also have
   \begin{align*}
     &f(S) + 2 \beta_{t-1} \sigma_{t-1}(S) \ge \\
     & \quad  \ucb_t(e_1 \mid \emptyset) \ge \ucb_t(e' \mid \emptyset) \ge \rho c(e').
   \end{align*}
   Therefore, we have
   \begin{align*}
     &f(S) + 2 \beta_{t-1} \sigma_{t-1}(S) \ge
      \rho \max \brl c(S), c(e')\brr \ge
      \frac{\rho}{2} \bl c(S) + c(e') \br\\
     &= \frac{\rho}{2} \sum_{j}\bl c_j(S) + c_j(e') \br
      \ge \frac{\rho}{2}.
   \end{align*}
   Here the last inequality holds because $S+e'$ does not satisfy the knapsack constraints.

   \textbf{Case Two}.
   This is the case when \algnamei{} terminates because
   for any element $e$ satisfying $\ucb_t(e \mid S),\ucb_t(e \mid \emptyset) \ge \rho c(e)$,
   $e$ satisfies knapsack constraints
   but $S + e$ does not satisfies the $k$-system constraint.
   We note that this case includes the case
   when there does not exist an element $e$ satisfying
   $\ucb_t(e \mid S),\ucb_t(e \mid \emptyset) \ge \rho c(e)$.

   We divide $C$ into two sets $C_{< \rho}$
   and $C_{\ge \rho}$, that is, we define
   \begin{align*}
   C_{< \rho} &= \brl e \in C \mid \exists i \text{ such that }
   \ucb_t(e \mid \setslc{S}{i}) < \rho c(e)
   \brr,\\
   C_{\ge \rho} &= C \setminus C_{< \rho}\\
   &= \brl
   e \in C \mid \forall i \quad
   \ucb_t(e \mid \setslc{S}{i}) \ge \rho c(e)
   \brr.
   \end{align*}
   Let $e \in C_{< \rho}$.
   Then on the event $\propevent$, we have
   $\Delta f(e \mid S) \le  \Delta f(e \mid \setslc{S}{i})
   \le \ucb_t(e \mid \setslc{S}{i}) < \rho c(e)$ for some $i$.
   Here the first inequality follows from submodularity and
   the second inequality follows from Proposition \ref{prop-ucb-gen}.
   Therefore, the following inequality holds:
   \begin{equation}
      \label{eq:crhorho-supp}
      \Delta f(C_{< \rho} \mid S) \le
      \sum_{e \in C_{< \rho}}  \Delta f(e \mid S)
      \le \sum_{e \in C_{< \rho}}  \rho c(e)
      \le l \rho.
   \end{equation}
   Here the first inequality follows from the submodularity of $f$
   and the second inequality follows from the fact that $C_{< \rho}$ is a feasible set.

   Next, we consider $C_{\ge \rho}$.
   We run Algorithm \ref{algo:greedy} on $S \cup C_{\ge \rho}$.
   By assumption of this case, Algorithm \ref{algo:greedy} returns $S$.
   Proposition \ref{prop:greedy-supp} implies that
   \begin{math}
      f(S) + 2\beta_{t-1}\sigma_{t-1}(S) \ge
      \frac{1}{k + 1}f(S \cup C_{\ge \rho}).
   \end{math}
   Rewriting, we have
   \begin{equation*}
      \Delta f(C_{\ge \rho} \mid S)  \le
      kf(S) + 2(k + 1)\beta_{t-1} \sigma_{t-1}(S).
   \end{equation*}
   By this, equation \eqref{eq:crhorho}, and submodularity of $f$,
   we have
   \begin{align*}
      &\Delta f(C \mid S)
      \le
      \Delta f(C_{< \rho} \mid S)  +
      \Delta f(C_{\ge \rho} \mid S)\\
      & \le l\rho + kf(S) + 2(k + 1)\beta_{t-1} \sigma_{t-1}(S).
   \end{align*}
   Noting that $\Delta f(C \mid S) = f(C \cup S) - f(S)$, we have
   \begin{equation*}
      f(S) + 2\beta_{t-1} \sigma_{t-1}(S)
      \ge \frac{1}{k + 1}f(S \cup C) - \frac{l\rho}{k + 1}.
   \end{equation*}
   This completes the proof.
\end{proof}

The following is a key lemma for the proof of our main Theorem.
\begin{lem}
   \label{lem:fantom}
   Let $\alpha = \frac{1}{(1+\varepsilon) \bl k + 2l + 1\br}$.
   Let $S$ be the set returned by \algname{} and $OPT$ the optimal feasible set.
   Then on the event $\propevent$, $S$ satisfies the following:
   \begin{equation*}
      f(S) + 4\beta_{t-1}\sigma_{t-1}(S) \ge \alpha f (OPT).
   \end{equation*}
\end{lem}
\begin{proof}
  By monotonicity and assumption,
  we have $\assumedbound \le \max_{e \in \ground} f(\{e\}) \le f(OPT)$.
  Here, we note that by assumptions any singleton $\{e\}$ for $e \in \ground$
  is a feasible set.
  By submodularity of $f$ and the assumption of $\assumedbound'$, we have
  \begin{equation*}
      f(OPT) \le \sum_{e \in OPT}f(\{e\}) \le |OPT| \assumedbound' \le |\ground|\assumedbound'.
  \end{equation*}
  We put $r = \frac{2}{k + 2l + 1}$
  and $r_{OPT} = r f(OPT)$.
  By bounds of $f(OPT)$ above,
  there exists $\rho$ in the while loop in \algname{}
  such that
  $(1 + \varepsilon)^{-1}r_{OPT} \le \rho \le r_{OPT}$.
  We denote such a $\rho$ by $\rho^{*}$.
  Moreover, if $\rho = r_{OPT}$, then
  we have $\rho/2 = \frac{1}{k + 1} f(OPT) - \frac{l\rho}{k + 1}$.
  By Proposition \ref{prop:sub-greedy-supp}, on the event $\propevent$, we have
  \begin{equation*}
     \mu_{t-1}(S_{\rho^{*}}) + 3 \beta_{t-1} \sigma_{t-1}(S_{\rho^{*}})
     \ge
     f(S_{\rho^{*}}) + 2 \beta_{t-1} \sigma_{t-1}(S_{\rho^{*}})
     \ge \frac{r_{OPT}}{2(1 + \varepsilon)}
     = \alpha f(OPT).
  \end{equation*}
  Here we denote by $S_\rho$ the returned set of \algnamei{} with threshold $\rho$.
  Because \algname{} returns a set with the largest UCB score in $\{S_{\rho}\}_{\rho}$,
  we have our assertion.
\end{proof}

\begin{rem}
   Suppose that the cardinality of any feasible solution
   is bounded by $m \in \ZZ_{>0}$.
   Then, by the proof of Lemma \ref{lem:fantom}, we see that
   in \algname{},
   the condition in the while loop $\rho \le r \assumedbound' |\ground|$
   can be replaced to $\rho \le r \assumedbound' m$.
   We obtain the same theoretical guarantee for this modified algorithm.
   Following FANTOM \parencite{mirzasoleiman2016fast},
   we use the condition $\rho \le r \assumedbound' |\ground|$.
\end{rem}
\subsection{PROOF OF THE MAIN THEOREM}
In this subsection, first, we introduce a lemma similar to
\parencite[Lemma 4]{chowdhury2017kernelized} and prove the main Theorem.
\begin{lem}
   \label{lem:sumsigma}
   Let $S_t$ be the set selected by \algname{} at time step $t$.
   We assume that $\lambda \ge m$.
   Then, we have
   \begin{equation*}
      \sum_{t = 1}^{T}\sigma_{t-1}(S_t \mid \phi_t) \le 2\sqrt{m T\gamma_{mT}}.
   \end{equation*}
\end{lem}
\begin{proof}
   First, suppose the kernel is a linear kernel.
   Then by \parencite[Lemma 5]{yue2011linear}, the definition of $\gamma_T$,
   and monotonicity of $T \mapsto \gamma_T$
   \parencite{krause2012near},
   we have
   \begin{math}
      \sum_{t=1}^T \sigma_{t-1}^2(S_t \mid \phi_t)
      \le
      4 \gamma_{mT}.
   \end{math}
   And the statement of the lemma follows from this inequality and the Cauchy-Schwartz inequality.
   Suppose that the kernel is a positive definite.
   Let $\{\psi_i\}_{i=1}^{\infty}$ be an orthonormal basis of the RKHS.
   Then for $x \in D$, we have $\kappa(x, \cdot) = \sum_{i=1}^{\infty}a_i(x)\psi_i$,
   where $a_i(x) \in \RR$.
   Therefore, for each $x, y \in D$, we have $\kappa(x, y) = \sum_{i=1}^{\infty}a_i(x)a_i(y)$.
   Since for fixed $T$,
   only finitely many values of $\kappa$ are appeared in the both sides of the inequality of the lemma,
   the statement for positive definite kernels can be proved by taking the limit of the inequality in the linear kernel case
   (see the proof of \parencite[Theorem 2]{chowdhury2017kernelized}).
\end{proof}

\begin{proof}[Proof of the main Theorem]
   Because $t \mapsto \beta_{t}$ is non-decreasing and by Lemma \ref{lem:fantom},
   we see that the $\alpha$-regret is bounded by
   \begin{math}
      4 \beta_T \sum_{t=1}^{T}\sigma(S_t \mid \phi)
   \end{math}
   with probability at least $1-\delta$ .
   Then the assertion of the main result follows from Lemma \ref{lem:sumsigma}.
   We also note that the second statement in the main article follows from the proof of
   \parencite[Theorem 8]{srinivas2010gaussian}.
\end{proof}
\begin{rem}
   We consider only stationary constraints for simplicity, i.e.,
   the constraints \eqref{eq:the-constraints} do not depend on time step $t$.
   However, it is clear from the proof that
   only $k$ and $l$ should be stationary and we can consider non-stationary constraints.
\end{rem}
\section{PROOF OF PROPOSITION \ref{prop:greedy-linear}}
In this section, we prove that MCSGreedy \parencite{yu2016linear} and LSBGreedy \parencite{yue2011linear}
perform arbitrary poorly in some situations.
We denote by $\regretmcs(T)$ and $\regretlsb(T)$
the $\alpha$-regret of MCSGreedy and LSBGreedy respectively.
\begin{prop}
   For any $\alpha > 0$,
   there exists cost $c$, $k$-system $\mathcal{I}$, a submodular function $f$,
   $T_0 > 0$ and a constant $C > 0$ such that
   with probability at least $1-\delta$,
   \begin{equation*}
      \regretmcs(T) > C T,
   \end{equation*}
   for any $T > T_0$.
   Moreover, the same statement holds for $\regretlsb(T)$.
\end{prop}
\begin{proof}
   We prove the statement only for MCSGreedy.
   We skip the proof for LSBGreedy because it is similar and simpler.
   We take a $k$-system constraint as the cardinality constraint
   $|S| \le m$, where we choose $m$ later.
   We consider $\ground$ as the disjoint union of $\ground_1$ and $\ground_2$,
   where $|\ground_1| = |\ground_2| = m$.
   We define cost $c: \ground \rightarrow \RR$ as
   \begin{equation*}
      c(e)  =
      \begin{cases}
         \frac{1}{m} & \text{if } e \in \ground_1,\\
         \frac{1}{m^2}
         & \text{if } e \in \ground_2.
      \end{cases}
   \end{equation*}
   For each $e \in \ground$, we define $v_e$ as follows:
   \begin{equation*}
      v(e) =
      \begin{cases}
         \frac{1}{m} & \text{if } e \in \ground_1,\\
         \frac{1 + \epsilon}{m^2} & \text{if } e \in \ground_2.
      \end{cases}
   \end{equation*}
   Here $\epsilon$ is a small positive number.
   We assume that the objective function $f: 2^{\ground} \rightarrow \RR$ is a modular function, i.e,
   $f(A) = \sum_{e \in A} v(e)$ for $A \subseteq \ground$.
   For $e \in \ground$, we define a function $\chi_e : 2^{\ground} \rightarrow \RR$ as
   \begin{equation*}
      \chi_e(A) = \begin{cases}
         1 & \text{if } e \in A,\\
         0 & \text{otherwise}
      \end{cases}
   \end{equation*}
   Then we have $f(A) = \sum_{e \in \ground} v(e) \chi_e(A)$.
   Therefore, this is the linear kernel case.
   Moreover, because feature vectors $[\chi_e(e' | A)]_{e \in \ground}$,
   for $e' \not \in A$ are orthogonal,
   this case can be reduced to the MAB setting.
   Therefore, we can use the UCB in Lemma 6 \parencite{abbasi2011improved}.

   Denote by $S_t$ the set selected by MCSGreedy and by $OPT$ the optimal solution.
   Because $c(e) \le 1/m$ for all $e \in \ground$, we have $|S_t| = m$.
   It is easy to see that $OPT = \ground_1$.
   Therefore we have $f(OPT) = 1$.

   We take sufficiently large $m$ so that
   \begin{equation*}
      \frac{(m-3)(1 + \epsilon)}{m^2} + \frac{3}{m} < \alpha.
   \end{equation*}
   For $e\in \ground$ and $1 \le t \le T$,
   we denote by $N_{e, t}$ the number of times $e$ is selected in time steps $\tau = 1, \dots, t$.
   Then by Lemma 6 \parencite{abbasi2011improved}, UCB of $v(e)$ for $e\in \ground$ is given as
   \begin{math}
      \mu_{e, t} + \sigma_{e, t}
   \end{math}
   where
   $\mu_{e, t}$ is the mean of feedbacks for $v(e)$
   and $\sigma_{e, t}$ is given as follows:
   \begin{equation*}
      \sigma_{e, t} = \sqrt{
         \frac{1 + N_{e, t}}{N_{e, t}^2}
         \bl
         1 + 2 \ln \bl
         d (1 + N_{e, t}) ^{1/2}/\delta
         \br
         \br},
   \end{equation*}
   where $d = 2m$.
   We note that at each round MCSGreedy selects elements
   with the
   largest modified UCB $(\mu_{e, t} + \sigma_{e, t})/c(e)$ except the first three elements.
   Suppose that the modified UCB of an element in $\ground_1$
   is greater than that of an element in $\ground_2$,
   that is,
   \begin{equation*}
      m(\mu_{e_1, t} + \sigma_{e_1, t}) > m^2 (\mu_{e_2, t} + \sigma_{e_2, t}).
   \end{equation*}
   By Lemma 6 \parencite{abbasi2011improved}, we have with probability at least $1-\delta$,
   \begin{equation*}
      m (v(e_1) + 2 \sigma_{e_1, t})
      \ge m (\mu_{e_1, t} + \sigma_{e_1, t})
      > m^2 (\mu_{e_2, t} + \sigma_{e_2, t}) \ge m^2 v(e_2).
   \end{equation*}
   Thus we have
   \begin{equation*}
      2 m \sigma_{e_1, t} > \epsilon.
   \end{equation*}
   Therefore, there exists $N(\delta, \epsilon, m) \in \ZZ_{> 0}$ such that
   \begin{equation*}
      N_{e_1, t} < N(\delta, \epsilon, m).
   \end{equation*}
   We assume that $T$ is sufficiently large compared to $N(\delta, \epsilon, m)$.
   We see that
   MCSGreedy selects
   at least
   $(m-3)T - m N(\delta, \epsilon, m)$
   elements from $\ground_2$.
   Thus with probability at least $1-\delta$, we have
   \begin{equation*}
      \sum_{t = 1}^{T}f(S_t) \le \bl
      (m - 3) T - m N(\delta, \epsilon, m)
      \br \frac{1 + \epsilon}{m^2}
      + \bl 3T + m N(\delta, \epsilon, m)
      \br \frac{1}{m}.
   \end{equation*}
   Because
   \begin{equation*}
      \alpha \sum_{t = 1}^T f(OPT) = \alpha T,
   \end{equation*}
   we have our assertion.
\end{proof}

\section{PARAMETERS IN THE EXPERIMENTS}
Because this is the linear kernel case,
by Theorem 5 \parencite{srinivas2010gaussian}, we have
$\gamma_t = O(k \ln t)$.
Thus we modify the definition of $\beta_t$ as
\begin{math}
   \beta_t = B + R_1 \sqrt{R_2 k \ln m_t + 1 + \ln(1/\delta)},
\end{math}
where $B, R_1$, and $R_2$ parameters.
For all datasets, we take $\nu = 0.01$ and $\nu' = 1.0$.
We take $\varepsilon = 0.3$, $\varepsilon = 1.0$, and $\varepsilon = 0.1$
for the news article recommendation dataset,
MovieLens100k, and the Million Song Dataset respectively.
We tune other parameters by using different users.
Explicitly, we take $\beta = 0.01, R_1 = 0.1, R_2 = 1.0, \lambda = 0.1$
for the news article recommendation example,
$\beta = 0.01, R_1 = 0.1, R_2 = 1.0, \lambda = 1.0$
for MovieLens100k, and
$\beta = 0.01, R_1 = 0.01, R_2 = 1.0, \lambda = 3.0$
for the Million Song Dataset.


\end{document}